\documentclass[journal]{IEEEtran}

\usepackage{graphicx}
\usepackage{url}
\usepackage{cite}
\usepackage{amsmath,amssymb}
\usepackage{booktabs}
\usepackage{array}
\usepackage{multirow}
\usepackage{placeins}
\usepackage{hyperref}

\newcommand{\bq}{\mathbf{q}}
\newcommand{\bk}{\mathbf{k}}
\newcommand{\bv}{\mathbf{v}}
\newcommand{\rope}{\textsc{RoPE}}

\newcommand{\swa}{\textsc{SWA}}
\newcommand{\calW}{\mathcal{W}}

\begin{document}

\title{Periodic and Harmonic Periodic {RoPE} \\ for Length Extrapolation in LLMs}

\author{Simin Huo, Ning Li
\thanks{Manuscript received 20 Nov, 2026. (Corresponding author: Ning Li.)}
\thanks{Ning Li is with the School of Automation and Intelligent Sensing, Shanghai Jiao Tong University, Shanghai, China (e-mail: ning\_li@sjtu.edu.cn).}
}

\markboth{IEEE Signal Processing Letters}
{Huo: Periodic and Harmonic Periodic RoPE}

\maketitle

\begin{abstract}
The ability to process long contexts is crucial for large language models (LLMs) to perform long-horizon tasks. Rotary position embedding (RoPE) gives every token a new rotation angle, so a model eventually meets angles it never saw in training, i.e., position exhaustion. We propose Periodic RoPE (P-RoPE), a positional encoding mechanism that operates in conjunction with sliding window attention. It repeats the first sliding-window-size \(W\) angles of the RoPE table, so the model needs only \(W\) distinct angles. P-RoPE works but introduces a wrap error at the window boundary. We further propose Harmonic Periodic RoPE (HP-RoPE), which removes the error by choosing frequencies that complete whole turns in \(W\) tokens. We compare models that share the same architecture, data, and schedule and differ only in the window-layer position encoding: standard RoPE, P-RoPE, and HP-RoPE. The results show that P-RoPE and HP-RoPE keep perplexity nearly flat from the training length \(768\) to \(16\)k, while RoPE, NTK, and YaRN rise. In a format-free Passkey test at \(16\)k, they remain at \(99\)--\(100\%\), against \(11\)--\(37\%\) for those RoPE interpolations. The code is available at \href{https://github.com/Cominder/miniwin}{https://github.com/Cominder/miniwin}. 
\end{abstract}


\IEEEpeerreviewmaketitle

\section{Introduction}
\IEEEPARstart{R}{otary} position embedding (RoPE)~\cite{su2024roformer} rotates each query and key by an angle that grows with the token index \(t\). However, when the sequence grows beyond the training length, the rotary angles enter a regime the model has never observed, i.e., position exhaustion, and perplexity rises~\cite{press2022alibi,chen2023extending}. Position interpolation~\cite{chen2023extending}, NTK-aware scaling, and YaRN~\cite{peng2023yarn} squeeze the same angles into the trained range. They delay the failure, but they do not remove it, because any scheme that gives each token its own index will sooner or later run out of indices. ALiBi~\cite{press2022alibi} and NoPE~\cite{haviv2022transformer,kazemnejad2023impact} do not use a rotation table, but they also give up RoPE's relative geometry.

\begin{figure}[!t]
\centering
\includegraphics[width=0.78\columnwidth]{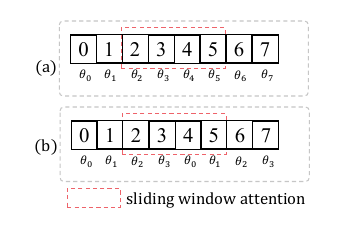}
\caption{(a)~Standard RoPE gives each index its own angle \(\theta_t\), as \(t\) grows and leaves the training range. (b)~P-RoPE repeats the same \(W\) angles in every window, so the cycle will never end no matter how long the sequence. Here \(W=4\).}
\label{fig:idea}
\end{figure}

Sliding-window attention (SWA)~\cite{beltagy2020longformer} bounds the visible range by a window \(W\). SWAN-GPT~\cite{puvvada2025swangpt} pairs local RoPE sliding windows with global layers that omit position encoding (NoPE). However, the absolute token index \(t\) in RoPE still grows and inevitably exceeds the training horizon. Inspired by the Iwin Transformer~\cite{huo2025iwin}, we realize that a local position encoding need not be unique across the entire sequence; it only needs to be unique inside each window. To this end, we propose \textbf{Periodic RoPE (P-RoPE)} (Fig.~\ref{fig:idea}), where the token at index \(t\) uses the angle at index \(\varphi(t)=t\bmod W\) rather than the absolute \(t\). Every local index lies in \(\{0,\ldots, W{-}1\}\) independently of sequence length; the model needs only \(W\) distinct angles, so position exhaustion never occurs.

P-RoPE simply repeats the first \(W\) angles of the RoPE table; it works but introduces a wrap error. When a query and a key sit on opposite sides of a repeat boundary, the layer sees their lag as \(d{-}W\) instead of actual lag \(d\). With the usual RoPE frequencies, this is a different rotation, so the layer applies the wrong relative angle to those pairs. Section~\ref{sec:method} derives the error, and Section~\ref{sec:wrap} measures it. Although the error exists, P-RoPE still trains and extrapolates. To eliminate it, we therefore propose \textbf{Harmonic Periodic RoPE (HP-RoPE)}, which replaces the RoPE frequencies by integer multiples of \(2\pi/W\), so that each channel completes a whole number of turns in one window. A shift by \(W\) tokens then becomes the identity rotation, and the wrap error vanishes. 

Our contributions are as follows.
\begin{itemize}

     \item We propose P-RoPE and HP-RoPE and show that P-RoPE works but incurs a wrap error, while HP-RoPE removes it by choosing frequencies whose periods divide the sliding window length.

     \item We build MiniWin, a hybrid of sliding-window layers and global NoPE layers. Only the window-layer encoding changes among RoPE, P-RoPE, and HP-RoPE. RoPE uses NTK, YaRN, and YaRN-mm at test time; P-RoPE and HP-RoPE are never interpolated.

    \item We show that P-RoPE and HP-RoPE keep perplexity and format-free Passkey nearly flat from the training length to \(16\)k, both after the \(768\)-token training and after the \(2\times\) length fine-tune, whereas RoPE with NTK, YaRN, or YaRN-mm does not.
    
\end{itemize}

\section{Method}
\label{sec:method}

\subsection{RoPE Recap}
\label{sec:rope}

For a token at position $t$ with head dimension $d_h$, RoPE splits \(d_h\) into \(d_h/2\) pairs. Pair \(i\) rotates at frequency
\begin{equation}
\theta_i = \frac{1}{\Theta^{2i/d_h}}, \quad i = 0, 1, \ldots,     \tfrac{d_h}{2}-1,
  \label{eq:freqs}
\end{equation}
where $\Theta$ is a base constant (typically $10^4$–$10^6$). A token at index \(t\) turns pair \(i\) by the angle \(t\theta_i\), which we write as the rotation $R_t \in \mathbb{R}^{d_h\times d_h}$, which is block-diagonal with $2{\times}2$ blocks
\begin{equation}
  \begin{pmatrix} \cos(t\,\theta_i) & -\sin(t\,\theta_i) \\
                  \sin(t\,\theta_i) &  \cos(t\,\theta_i) \end{pmatrix}.
  \label{eq:rot}
\end{equation}
A query $\bq_t$ and key $\bk_s$ are rotated before the dot product:
\begin{equation}
  \bq_t^{\top}\bk_s = (R_t\tilde{\bq}_t)^{\top}(R_s\tilde{\bk}_s)
                     = \tilde{\bq}_t^{\top}R_{t-s}\tilde{\bk}_s,
  \label{eq:rope_score}
\end{equation}
where the last equality exploits $R_t^{\top}R_s = R_{t-s}$. The attention score therefore depends only on the \emph{relative offset} $t-s$, making RoPE translation-equivariant in principle. In practice, however, the model generalizes reliably only to offsets seen during training. During inference, the absolute index $t$ grows without bound. At a step $t > T_{\max}$ (the training context length), the rotation $R_t$ is built from angles $t\,\theta_i$ that the model has never been trained on, which degrades quality.

\subsection{P-RoPE}
Instead of using the absolute position $t$, P-RoPE works with the causal \swa{} of window size $W$, where each query at position $t$ attends only to positions in $\calW(t) = \{t-W+1,\ldots, t\}$. We define the \emph{periodic position index}.
\begin{equation}
  \varphi(t)= t \bmod W
  \label{eq:periodic_id}
\end{equation}
and apply \rope{} using $R_{\varphi(t)}$ rather than $R_t$, so we have

\begin{equation}
  \begin{pmatrix} \cos(\varphi(t)\,\theta_i) & -\sin(\varphi(t)\,\theta_i) \\
                  \sin(\varphi(t)\,\theta_i) &  \cos(\varphi(t)\,\theta_i) \end{pmatrix}.
  \label{eq:prot}
\end{equation}

The rotation table $\{R_0, R_1, \ldots, R_{W-1}\}$ is precomputed once and indexed cyclically throughout generation. Because all indices stay in $\{0,\ldots, W-1\}$, the model \emph{never} accesses a position index outside its training range, regardless of the absolute sequence length.

\subsection{The wrap error of P-RoPE}
\label{sec:error}

RoPE scores a query at index \(q\) against a key at index \(k\) by
\begin{equation}
\alpha_{qk}
=\bigl(R_{q}\bq\bigr)^{\top}\bigl(R_{k}\bk\bigr)
=\bq^{\top} R_{q-k}\bk.
\label{eq:score}
\end{equation}
Only the lag \(d=q-k\) enters the rotation. P-RoPE keeps the same bilinear form but substitutes \(\varphi(t)=t\bmod W\),
\begin{equation}
\alpha_{qk}^{\mathrm{P}}
=\bq^{\top} R_{\varphi(q)-\varphi(k)}\bk.
\label{eq:scorep}
\end{equation}
Write \(q=nW+a\) and \(k=mW+b\) with remainders \(a,b\in\{0,\ldots,W-1\}\). Inside a causal window of width \(W\), there are two cases.

\textit{Same period.} If \(n=m\), then \(\varphi(q)-\varphi(k)=a-b=d\), so~\eqref{eq:scorep} equals~\eqref{eq:score}. The table is being reused, but the relative angle is the true lag.

\textit{Across a boundary.} If \(n=m+1\), the query has wrapped to a small remainder while the key has not. Then
\begin{equation}
\varphi(q)-\varphi(k)=a-b=d-W,
\label{eq:dminusw}
\end{equation}
and the layer uses \(R_{d-W}\) instead of \(R_d\). Rotations add, so
\begin{equation}
R_{d-W}=R_{d}\,R_{-W}.
\label{eq:extra}
\end{equation}
The extra factor \(R_{-W}\) is the identity if and only if \(W\theta_i\) is a multiple of \(2\pi\) on every channel \(i\). For the geometric frequencies \(\theta_i=\Theta^{-2i/d_h}\) this fails, so \(R_{-W}\ne I\) (Fig.~\ref{fig:method}a,b). Compactly, the encoded lag is
\begin{equation}
\Delta\varphi(q,k)
=d-W\cdot\mathbf{1}\bigl[\lfloor q/W\rfloor\neq\lfloor k/W\rfloor\bigr].
\label{eq:deltaphi}
\end{equation}
For a fixed lag \(d\), a fraction \(d/W\) of query positions fall in the second case: \(d=1\) wraps once every \(W\) steps, and \(d=W/2\) wraps half of the time. When \(d>W/2\), one also has \(|d-W|<d\), so a wrap-side key is encoded as if it were nearer than it is. That is the pattern in Fig.~\ref{fig:lag}.

\begin{figure*}[!t]
\centering
\includegraphics[width=\textwidth]{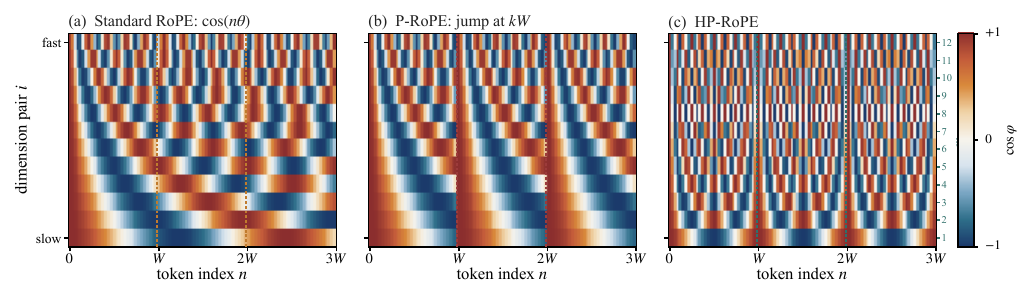}
\caption{Cosine of the rotation angle (color) on frequency pairs, fast pairs at the top. (a)~Standard RoPE, \(\cos(n\theta)\). (b)~P-RoPE, \(\cos((n\bmod W)\theta)\): the pattern jumps at every multiple of \(W\) because \(R(W)\ne I\). (c)~HP-RoPE uses whole turns per window, so \(R(W)=I\) and the pattern repeats without a jump.}
\label{fig:method}
\end{figure*}

\subsection{HP-RoPE}
\label{sec:hprope}

The extra factor in~\eqref{eq:extra} vanishes once \(R_{W}=I\). We set each frequency to an integer number of turns per window,
\begin{equation}
\theta'_i=\frac{2\pi m_i}{W},\qquad m_i\in\mathbb{Z}.
\label{eq:harm}
\end{equation}
So $R'_{W}=R'_{-W}=I$. Then \(R'_{d-W}=R'_{d}R'_{-W}=R'_{d}\), so both cases of~\eqref{eq:deltaphi} produce the same relative rotation (Fig.~\ref{fig:method}c). The two offsets \(d\) and \(d{-}W\) are no longer different operators. The equality holds in the rotation table itself, so the model does not have to learn the wrap.

Let \(n_p=d_h/2\). We set 
\begin{equation}
m_i=n_p-i,\qquad i=0,\ldots,n_p-1.
\label{eq:ladder}
\end{equation}
In our configuration \(W{=}256\) and \(d_h{=}96\), so \(n_p{=}48\) and \(m_i\in\{48,47,\ldots,1\}\subset\{1,\ldots,W-1\}\). Pair \(0\) is fastest: one turn takes
$\frac{2\pi}{\theta'_0}=\frac{W}{m_0}=\frac{256}{48}\approx 5.33
$ tokens. The slowest pair has \(m_{n_p-1}{=}1\) and completes one turn in \(W{=}256\) tokens.

The same identity is what the frozen test in Section~\ref{sec:wrap} checks. Keep \(\bq,\bk,\bv\) fixed and replace \(\varphi\) by a shifted table
\begin{equation}
\varphi_s(t)=(t+s)\bmod W,\qquad s\in\{0,1,\ldots,W-1\}.
\label{eq:shift}
\end{equation}
If \(R_{W}=I\), then \(R_{\varphi_s(t)}=R_{t+s}=R_{t}R_{s}\). Query and key pick up the same extra rotation \(R_{s}\), which cancels in~\eqref{eq:score}, so the logits cannot change. If \(R_{W}\ne I\), wrapping and adding \(s\) are not the same operation: which pairs fall under~\eqref{eq:dminusw} depends on \(s\), and the logits move. Fig.~\ref{fig:shift} shows they move for every \(s\neq 0\) under P-RoPE and stay at floating-point noise under HP-RoPE.

\section{Experiments}
\label{sec:exp}

\subsection{Setup}

\textbf{Training.} Based on MiniMind3~\cite{minimind}, we trained the MiniWin model, which includes RoPE, P-RoPE, and HP-RoPE variants. All models use the same data and schedule. They are trained at length \(768\) on about \(8.4\)M documents after a \(5\%\) hold-out (AdamW, bf16). The documents are packed in chunks, where documents are concatenated with EOS separators, and each chunk starts with one BOS. A second continue-pretraining (CPT) stage uses about \(10\%\) of the packed tokens at length \(1536\) (\(768{\times}2\)). RoPE CPT is a naive length extension: the original frequencies on indices \(0,\ldots,1535\), with no NTK or YaRN while training. P-RoPE and HP-RoPE keep the same window table and are never interpolated. NTK, YaRN, and YaRN-mm~\cite{peng2023yarn} are applied only at test time.

\textbf{Perplexity.} We report absolute perplexity at lengths \(512\) to \(16\)k on two streams: a packed hold-out of the pre-training data (\(32\) chunks), and a contiguous novel, \emph{Journey to the West}, excluded from training (about \(760\)k tokens, \(24\) chunks). The first stream tests in-domain language modeling, while the novel tests long-range use of context. 

\textbf{Passkey.} We use the English template of Landmark Attention~\cite{mohtashami2023landmark} and a parallel Chinese template, with exact token length and uniform needle depth. Free generation mixes the format following retrieval, so we score the true \(5\)-digit key against \(7\) random keys by \(\log p(\text{key}\mid\text{prompt})\) and report top-1 accuracy (chance \(12.5\%\)), with \(200\) trials per length. 

\begin{figure}[!t]
\centering
\includegraphics[width=0.6\columnwidth, angle=-90]{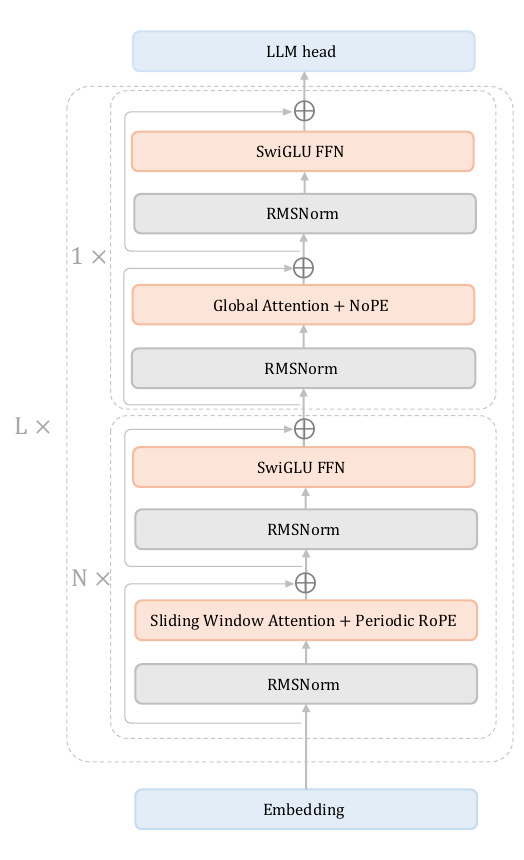}
\caption{Overall MiniWin architecture. \(L{=}2\) hybrid blocks, each with \(N{=}3\) sliding-window attention layers with position encoding and one global attention layer with NoPE. A SwiGLU FFN follows every attention.}
\label{fig:arch}
\end{figure}

\begin{table}[!t]
\centering
\caption{Model configuration.}
\label{tab:cfg}
\setlength{\tabcolsep}{1.8pt}
\renewcommand{\arraystretch}{1.05}
\begin{tabular*}{\columnwidth}{@{\extracolsep{\fill}}lc@{}}
\toprule
 Configuration & Value \\
\midrule
Hidden size / layers & \(768\) / \(8\) \\
Attention pattern & SSSL\(\times 2\) (\(6\)S+\(2\)L) \\
Sliding window size & \(W{=}256\) \\
Q / KV heads & \(8\) / \(4\) \\
Head dimension \(d_h\) & \(96\) \\
FFN (SwiGLU) & \(2048\) \\
S-layer PE & P-/HP-RoPE/RoPE\\
L-layer PE & NoPE \\
$\Theta$ & $10^{6}$ \\
Vocabulary (BPE) & \(6400\) \\
Train length & \(768\) (\(+\mathrm{CPT}\,1536\)) \\
\bottomrule
\end{tabular*}
\end{table}

\subsection{Length extrapolation}

\begin{figure*}[!t]
\centering
\includegraphics[width=\textwidth]{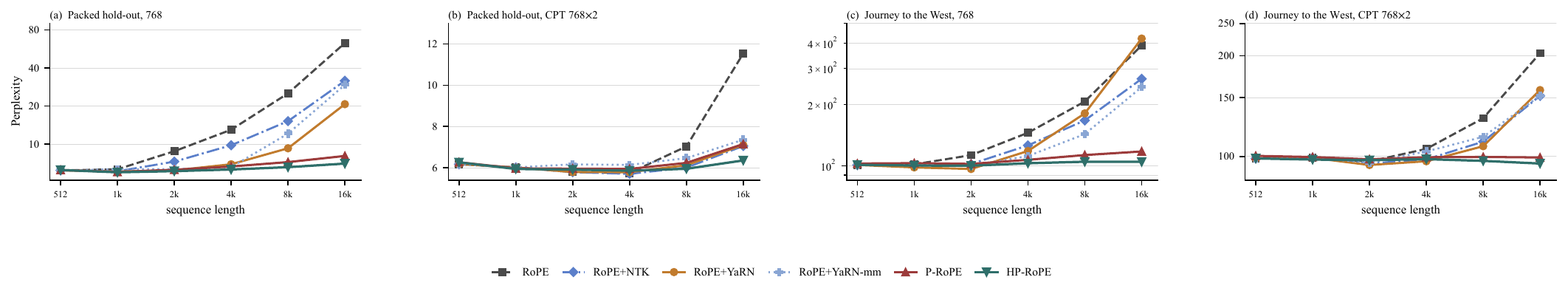}\\[-2.3em]
\includegraphics[width=\textwidth]{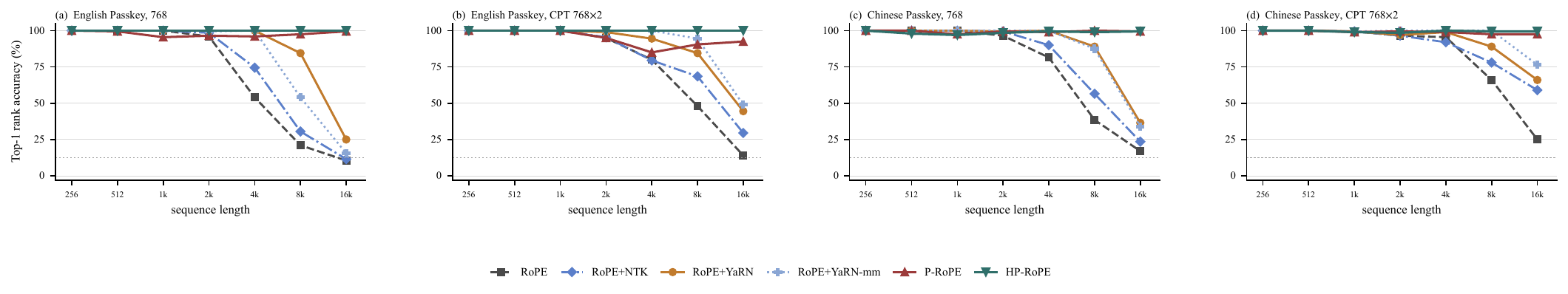}
\caption{Length extrapolation of the \(768\)-token model (a,c) and CPT \(768{\times}2\) (b,d). Top: absolute perplexity on packed hold-out (a,b) and on \emph{Journey to the West} (c,d). Bottom: format-free Passkey top-1 accuracy (\%; chance \(12.5\%\)) in English (a,b) and Chinese (c,d).}
\label{fig:main}
\end{figure*}

\textbf{Perplexity.} Fig.~\ref{fig:main} (top) shows the curves. On the packed hold-out at \(16\)k, RoPE rises to \(62.9\) (\(768\)) and \(11.5\) (CPT), and YaRN reaches \(20.7\) and \(7.1\). P-RoPE stays at \(8.07\) and \(7.15\), and HP-RoPE at \(7.15\) and \(6.35\). The gap is wider on the novel: RoPE reaches \(390\) and \(204\) at \(16\)k, YaRN \(421\) and \(158\), while P-RoPE is at \(118\) and \(99\), and HP-RoPE at \(113\) and \(98\). Interpolation helps on the hold-out but does not close the gap on the novel. HP-RoPE has the flattest curve in all four panels, and it costs nothing near the training length.

\textbf{Passkey.} Fig.~\ref{fig:main} (bottom) shows rank accuracy. At \(16\)k, P-RoPE and HP-RoPE reach \(99\%\) and \(100\%\) in English, and \(100\%\) and \(99.5\%\) in Chinese (\(768\)-token models), against \(13\%\) and \(14\%\) for RoPE, and at most \(22\%\) and \(42\%\) for YaRN. CPT keeps the same ordering. HP-RoPE does not beat P-RoPE on Passkey by more than a few points at any length. Retrieval therefore comes from the hybrid stack with a periodic local table, and the benefit of removing the wrap error shows mainly in perplexity and in the window-layer diagnostics above.

\subsection{Does the wrap error matter in trained models?}
\label{sec:wrap}

\textbf{Frozen origin shift.} Equation~\eqref{eq:shift} says that if the score depends only on the lag \(d\), moving the repeat boundary by \(s\) tokens must leave the logits unchanged. We freeze the pre-RoPE queries, keys, and values of a trained window layer and re-apply the rotation table with \(\varphi_s\). Let \(\ell^{(s)}\) be the attention logits and \(a^{(s)}=\mathrm{softmax}(\ell^{(s)})\). Over the six window layers we report \(|\Delta\mathrm{logits}|_{\max}=\max\bigl|\ell^{(s)}_{h,t,j}-\ell^{(0)}_{h,t,j}\bigr|\) and \(\mathrm{TV}_{\max}=\max\,\overline{\mathrm{TV}}_{h,t}\) with \(\mathrm{TV}(a,a')=\tfrac12\sum_j|a_j-a'_j|\). TV is \(0\) when two attention rows are identical and \(1\) when they have disjoint support.

Fig.~\ref{fig:shift} sweeps \(s\) from \(1\) through \(255\). For P-RoPE, \(|\Delta\mathrm{logits}|_{\max}\) stays near \(21\) for every \(s\neq 0\) (\(21.10\) already at \(s=1,2,3,4\); \(23.11\) at \(s=128\)). \(\mathrm{TV}_{\max}\) stays between \(0.31\) and \(0.81\). For HP-RoPE both curves sit at \(10^{-6}\) to \(10^{-7}\).

\begin{figure}[!t]
\centering
\includegraphics[width=\linewidth,keepaspectratio]{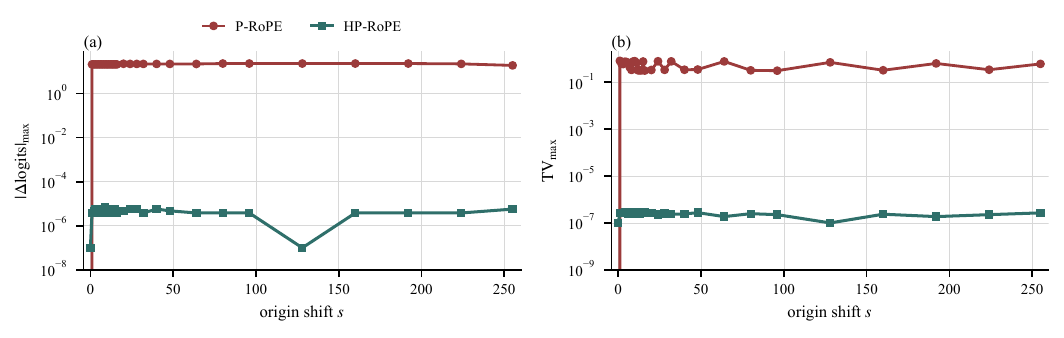}
\caption{Frozen origin shift on the \(768\)-token checkpoints. (a)~Largest logit change vs.\ shift \(s\). (b)~Largest total variation of the attention row.}
\label{fig:shift}
\end{figure}

\textbf{Two-needle test.} The frozen test shows that the wrap error changes the logits, but not whether trained layers use it. We therefore plant a true key at lag \(160\) and a distractor at lag \(24\), and change only the prompt length, which moves the boundary. At length \(768\), both keys lie in the query's period. At length \(833\), the true key lies across the boundary and the distractor does not. We read the share of window-layer attention that falls on the true key, relative to the sum of true and distractor, averaged over six layers and \(100\) trials. For P-RoPE the share falls from \(7.5\%\) to \(5.4\%\) when the key crosses the boundary, and the drop grows to \(3.1\) points after CPT. For HP-RoPE it stays at \(7.7\%\) and \(7.6\%\), almost not changing.

\begin{table}[!t]
\centering
\caption{Window-layer attention on the true key (\% of true+distractor). In period: length \(768\), both keys in one period. Across wrap: length \(833\), true key across the boundary.}
\label{tab:needle}
\setlength{\tabcolsep}{3pt}
\renewcommand{\arraystretch}{1.05}
\begin{tabular*}{\columnwidth}{@{\extracolsep{\fill}}lccc@{}}
\toprule
 & in period & across wrap & \(\Delta\) \\
\midrule
P-RoPE, \(768\) & \(7.5\) & \(5.4\) & \(2.1\) \\
HP-RoPE, \(768\) & \(7.7\) & \(7.6\) & \(0.1\) \\
P-RoPE, CPT & \(9.1\) & \(6.0\) & \(3.1\) \\
HP-RoPE, CPT & \(7.3\) & \(7.2\) & \(0.1\) \\
\bottomrule
\end{tabular*}
\end{table}

\textbf{Attention over lag.} Fig.~\ref{fig:lag} shows mean window attention as a function of lag at three query phases. At \(\varphi{=}255\), no key lies across the boundary, and the two models look alike. At \(\varphi{=}254\) one key does, and at \(\varphi{=}128\) half of the window does. The last query then places \(1.3\%\) (\(\varphi{=}254\)) and \(11.1\%\) (\(\varphi{=}128\)) of its attention on keys across the boundary in P-RoPE, against \(0.7\%\) and \(5.4\%\) in HP-RoPE. These keys have lag larger than \(\varphi\), and P-RoPE attends to them more than HP-RoPE does, as \eqref{eq:dminusw} predicts: when \(d>W/2\), \(|d-W|<d\).

\begin{figure}[!t]
\centering
\includegraphics[width=\linewidth]{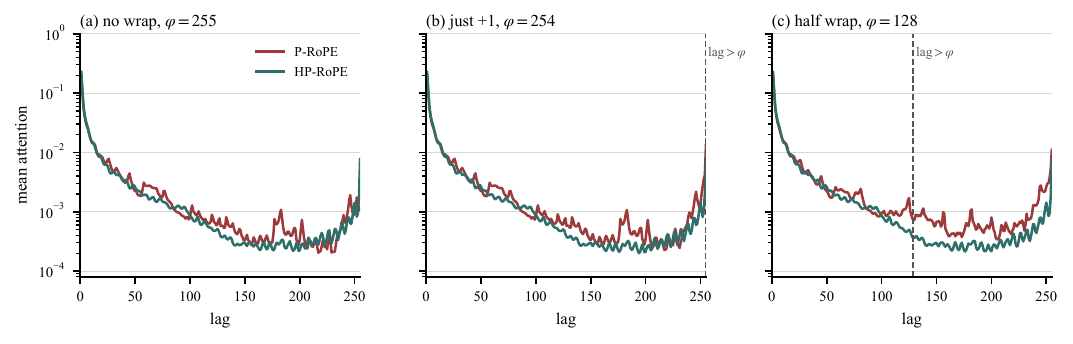}
\caption{Mean window attention vs.\ lag (log scale), averaged over heads, the six window layers. Phases of the last query: (a)~\(\varphi{=}255\), no key across the boundary; (b)~\(\varphi{=}254\), one key; (c)~\(\varphi{=}128\), half the window. The dashed line marks lag \(>\varphi\), the keys across the boundary.}
\label{fig:lag}
\end{figure}

\textbf{Window size \(W\).} Table~\ref{tab:wscan} compares models trained at length \(768\) with \(W\in\{64,128,256\}\). PPL changes little with \(W\). Wrap stays large for P-RoPE and near zero for HP-RoPE.

\begin{table}[!t]
\centering
\caption{PPL on packed hold-out. Wrap: mean \(|\Delta\mathrm{logits}|\) at frozen \(s{=}1\).}
\label{tab:wscan}
\setlength{\tabcolsep}{2.2pt}
\renewcommand{\arraystretch}{1.05}
{\footnotesize
\begin{tabular*}{\columnwidth}{@{\extracolsep{\fill}}lcccccc@{}}
\toprule
\(W\) & P \(512\) & HP \(512\) & P \(4\)k & HP \(4\)k & wrap P & wrap HP \\
\midrule
\(64\) & \(6.34\) & \(6.19\) & \(6.53\) & \(6.22\) & \(5.98\) & \(10^{-6}\) \\
\(128\) & \(6.36\) & \(6.27\) & \(6.47\) & \(6.10\) & \(8.83\) & \(10^{-6}\) \\
\(256\) & \(6.25\) & \(6.16\) & \(6.66\) & \(6.27\) & \(7.60\) & \(10^{-6}\) \\
\bottomrule
\end{tabular*}
}
\end{table}

\section{Conclusion}

We have presented P-RoPE and HP-RoPE, two window-periodic positional encodings for sliding-window attention.
P-RoPE reuses the first \(W\) RoPE angles; HP-RoPE further selects frequencies whose periods divide the window length, removing the wrap error at the window boundary.
Neither method requires careful interpolation for length extrapolation. Under a comparison that varies only the window-layer encoding, both remain stable beyond the training length, whereas RoPE and its interpolations do not. A next step is to validate P-RoPE and HP-RoPE on larger models and on broader benchmarks.

\bibliographystyle{IEEEtran}
\bibliography{refs}

\end{document}